\documentclass{article}
\usepackage[utf8]{inputenc}
\usepackage[T1]{fontenc}
\usepackage{microtype}
\usepackage{graphicx}
\usepackage{subfigure}
\usepackage{booktabs} 
\usepackage{hyperref}

\usepackage[accepted]{icml2021}

\usepackage{amsmath}
\usepackage{amssymb}
\usepackage{amsthm}
\usepackage{url}

\newtheorem{theorem}{Theorem}
\newtheorem{lemma}[theorem]{Lemma}

\newcommand{\C}{\mathcal{C}}
\newcommand{\A}{\mathcal{A}}
\newcommand{\E}{\mathbb{E}}
\newcommand{\R}{\mathbb{R}}
\newcommand{\cG}{\mathcal{G}}
\newcommand{\B}{\mathcal{B}}

\DeclareMathOperator{\ALG}{ALG}
\DeclareMathOperator{\cut}{cut}
\DeclareMathOperator{\err}{dis}
\DeclareMathOperator{\agr}{agr}
\DeclareMathOperator{\Lap}{Lap}
\newtheorem{corollary}[theorem]{Corollary}
\newtheorem{observation}[theorem]{Observation}
\newtheorem{proposition}[theorem]{Proposition}
\theoremstyle{definition}
\newtheorem{definition}[theorem]{Definition}

\def\final{1}
\ifnum\final=0
\newcommand{\fullonly}[1]{}
\newcommand{\shortonly}[1]{#1}
\else
\newcommand{\fullonly}[1]{#1}
\newcommand{\shortonly}[1]{}
\fi

\icmltitlerunning{Differentially Private Correlation Clustering}

\begin{document}

\twocolumn[
\icmltitle{Differentially Private Correlation Clustering}
\icmlsetsymbol{equal}{*}

\begin{icmlauthorlist}
\icmlauthor{Mark Bun}{equal,bu}
\icmlauthor{Marek Eliáš}{equal,cwi}
\icmlauthor{Janardhan Kulkarni}{equal,msr}
\end{icmlauthorlist}

\icmlaffiliation{bu}{Boston University, Boston}
\icmlaffiliation{cwi}{Centrum Wiskunde \& Informatica, Amsterdam}
\icmlaffiliation{msr}{Microsoft Research, Redmond}
\icmlcorrespondingauthor{Marek}{afds@fasd.df}
\icmlkeywords{Differential Privacy, Correlation Clustering}

\vskip 0.3in
]

\printAffiliationsAndNotice{\icmlEqualContribution}

\begin{abstract}
Correlation clustering is a  widely used technique in unsupervised machine learning.
Motivated by applications where individual privacy is a concern, we initiate the study of differentially private correlation clustering.
We propose an algorithm that achieves subquadratic additive error compared to the optimal cost. 
In contrast, straightforward adaptations of  existing non-private algorithms all lead to a trivial quadratic error.
Finally, we give a lower bound showing that any pure differentially private algorithm for correlation clustering requires additive error of $\Omega(n)$.
\end{abstract}

\section{Introduction}

Correlation clustering is a fundamental task in unsupervised machine learning. Given a set of objects and information about whether each pair is ``similar'' or ``dissimilar,'' the goal is to partition the objects into clusters that are as consistent with this information as possible. The correlation clustering problem was introduced by \citet{BansalBC02} and has since received significant attention in both the theoretical and applied machine learning communities. It has been successfully applied in numerous domains, being used to perform co-reference resolution \citep{ZHENG20111113}, image segmentation \citep{kim2014image}, gene clustering~\citep{genes}, and cancer mutation analysis \citep{CCcancer}.

In many important settings, the relationships between the objects we wish to cluster may depend on sensitive personal information about individuals. For example, suppose we wish to perform entity resolution on a collection of companies by clustering those that likely belong to the same organizational structure. For example, we would like to group Amazon Marketplace, Amazon Fresh, and less obviously, Twitch into the same cluster. The identities of the companies are public information, but our information about the relationships between them may come from sensitive information, e.g., from transaction records and personal communications. Moreover, the information we have on a relationship could be dramatically affected by individual data records: for instance, a single email with the phrase, ``This matter has been escalated from General Electric to the board meeting of the Sheinhardt Wig Corporation'' would be highly indicative of a similarity between two companies that all other evidence may point toward being extremely dissimilar.
We initiate the study of \emph{differentially private correlation clustering}, using edge level privacy, to address individual privacy concerns in such scenarios.

In this work, we design efficient algorithms for several different formulations of the private correlation clustering problem. In all of these variants, objects are represented as a set of vertices $V$ in a graph. Two vertices are connected by an edge with either positive or negative label if we have information about their similarity, e.g., coming from the output of a comparison classifier. Of special interest is when the graph is \emph{complete}, i.e., we possess similarity information about all pairs of vertices, though we also consider the problem for general graph topologies. Edges in the graph may be either \emph{unweighted} or \emph{weighted},
where the weight of an edge can be viewed as the confidence with which its endpoints are similar (positive label) or dissimilar
(negative label). 

A perfect clustering of the graph would be a partition of $V$ into clusters $C_1, \dots, C_k$ such that all positive-labeled edges connect vertices in the same cluster and all negative-labeled edges connect vertices in different clusters. In general, similarity information may be inconsistent, so no such clustering may exist. Thus, we define two problems corresponding to optimizing two related objective functions. In the Minimum Disagreement (MinDis) problem, we aim to minimize the total weight of violated edges,  i.e., the sum of the weights of positive edges that cross clusters plus the sum of the weights of negative edges within clusters.  The Maximum Agreement (MaxAgr) problem is to maximize the sum of weights of positive edges within clusters plus the sum of weights of negative edges across clusters. Note that the number of clusters $k$ is generally not specified in advanced. For both problems, we study algorithms with mixed multiplicative and additive guarantees, i.e., algorithms that report clusterings with MinDis $\le \alpha \cdot \operatorname{OPT} + \beta$ or MaxAgr $\ge \alpha \cdot \operatorname{OPT} - \beta$.

\subsection{Our results and Techniques}
All the formulations of private correlation clustering we consider admit algorithms with low additive error (and no multiplicative error) based on the \emph{exponential mechanism}~\cite{McSherryT07}, a generic primitive for solving discrete optimization problems. Observing that both the MinDis and MaxAgr objective functions have global sensitivity $1$, instantiating the exponential mechanism over the search space of all possible partitions gives an algorithm with additive error $O(n \log n)$, where $n$ is the number of vertices.

Our first result shows that the error achievable by the exponential mechanism is nearly optimal for path graphs.

\begin{theorem}[Informal]
\label{thm:lowerbound}
Any $\epsilon$-differentially private algorithm for correlation clustering on paths with either the MinDis or MaxAgr objective function has an additive error of $\Omega(n)$.
\end{theorem}

This lower bound raises the natural question of whether we can sample efficiently from the exponential mechanism for correlation clustering. 
Unfortunately, we do not know if this is possible.
Moreover, the APX-hardness of both MinDis and MaxArg \citep{BansalBC02} suggests that even non-private algorithms require multiplicative error larger than one.
We aim to design polynomial-time algorithms achieving a comparable additive error to the exponential mechanism and with minimal multiplicative error.

A natural place to start is to modify the existing algorithms for the problem from the non-private setting.
Correlation clustering has been studied extensively in the approximation algorithms, online algorithms, and machine learning communities \citep{BansalBC02,mathieu_et_al:LIPIcs:2010:2486,NIPS2015_b53b3a3d},
and many algorithms are known with strong provable guarantees.
Consider the algorithm of \citet{AilonCN05}, which solves the MinDis problem on unweighted complete graphs with multiplicative error at most $3$.
It is an iterative algorithm that proceeds as follows. 
In each iteration, pick a random vertex to be a pivot.
All the neighbors of the pivot vertex connected to it with a positive edge are added to form a new cluster,
and removed from the graph.
The process is repeated until there are no more vertices left in the graph.
\citet{AilonCN05} show via a careful charging argument on the triangles of the graph that this produces a 3-approximation to the optimal solution.

One way to make the algorithm of \citet{AilonCN05} differentially private is to use the exponential mechanism with an appropriate scoring function -- a strategy reminiscent of the approach used in submodular maximization problem \cite{Mitrovic2017DifferentiallyPS} -- or  to use the randomized response algorithm to decide whether a neighbor of the pivot vertex should be added to the new cluster in each iteration.
However, these strategies could lead to $\Omega(n^2)$ error as can be seen by running the algorithm of \citet{AilonCN05} on a complete graph with all edges having negative labels.
We hit similar roadblocks for other approaches to correlation clustering  based on metric space embedding \cite{ChawlaMSY15}.
Despite our efforts, we could not make existing algorithms for correlation clustering achieve any non-trivial sub-quadratic error.

Our main result is an efficient differentially private algorithm for correlation clustering with sub-quadratic error.
The following theorem is our main technical contribution.

\begin{theorem}
\label{thm:upperboundforCC}
There is an $\epsilon$-DP algorithm for correlation clustering on complete graphs guaranteeing 
      $$\err(\C, G) \leq 2.06 \err(\C^*,G) + O\left(\frac{n^{1.75}}{\epsilon}\right),$$
where $\err(\C^*,G)$, $\err(\C,G)$ denote the MinDis costs of an optimal clustering and of our algorithm's clustering, respectively. Moreover, there is an $(\epsilon, \delta)$-DP algorithm for general graphs with 
      $$\err(\C, G) \leq O(\log n) \err(\C^*,G) + O\left(\frac{n^{1.75}}{\epsilon}\right).$$  
\end{theorem}

The multiplicative approximation factors in the above theorem match the best known approximation factors
in the non-private setting~\citep{BansalBC02,ChawlaMSY15}, which are known to be near optimal.
Our results also extend to other objective functions such as MaxAgr, and to other variants of the problem where one requires that the number of clusters output by the algorithm is at most some small constant $k$. 
We discuss the extensions of our main theorem to these settings in Section \ref{sec:alg}.

The techniques we use to prove Theorem \ref{thm:upperboundforCC} are based on  private synthetic graph release.
We use recent work of \citet{GRU12} and \citet{2019_DP-cut} to release synthetic graphs preserving all of the cuts on the set of all positive edges, and on the set of all negative edges.
We then appeal to non-private approximation algorithms to obtain good clusterings on these synthetic graphs.
Finally, our sub-quadratic error bound is obtained by coarsening the clusters produced by our algorithm, and establishing a structural property that  any instance of the problem
has a good solution with a small number of clusters. 

There are relatively few problems in graph theory that admit accurate differentially private algorithms. Our result adds to this short list, by giving the first non-trivial bounds for the problem. 
However, we believe that there is a DP-correlation clustering algorithm that runs in polynomial time and matches the additive error of the exponential mechanism. 
This is an exciting open problem given the prominent position correlation clustering occupies both in theory and practice. 

%

\section{Preliminaries}

\subsection{Correlation clustering}
We survey the basic definitions and most important results on
correlation clustering in the non-private setting.

\begin{definition}
Let $G$ be a weighted graph with non-negative weights and let $E^+$ and $E^-$ denote
the sets of edges with positive and negative labels, respectively.
Given a clustering $\C = \{C_1, \dots, C_k\}$, we say that an edge in $e\in E^+$ agrees with $\C$
if its both endpoints belong to the same cluster. Similarly,
$e\in E^-$ agrees with $\C$ if its endpoints belong to different clusters.
We define the {\em agreement} $\agr(\C,G)$ between $\C$ and $G$
as the total weight of edges agreeing with $\C$,
and the {\em disagreement} $\err(\C,G)$ as the total
weight of edges which do not agree with $\C$.
\end{definition}

In MinDis problem,  we want to find a clustering $\C$ which minimizes
$\err(\C,G)$ on the input graph $G$.
Similarly, in MaxAgr, we want to maximize
$\agr(\C,G)$.

Correlation clustering is known to be much easier on unweighted complete graphs,
where there are several constant-approximation algorithm for MinDis
\citep{BansalBC02,AilonCN05,ChawlaMSY15} and a PTAS for MaxAgr \citep{BansalBC02}.

\begin{proposition}[\citet{ChawlaMSY15}]
\label{alg:chawla}
There is a polynomial-time algorithm for MinDis on unweighted complete graphs
with approximation ratio $2.06$.
\end{proposition}

\begin{proposition}[\citet{BansalBC02}]
\label{alg:bansal}
For every constant $\gamma > 0$, there is a polynomial-time algorithm
for MaxAgr on unweighted complete graphs with approximation ratio
$(1-\gamma)$.
\end{proposition}
\shortonly{\vspace{-2mm}}

On weighted graphs (with edges of weight 0 being especially problematic; see
\citet{JafarovKMM20}), there are algorithms achieving an approximation ratio of
$O(\log n)$ for MinDis \citep{DemaineEFI06,CharikarGW03}
and 0.7666 for MaxAgr \citep{Swamy04,CharikarGW03}.

\begin{proposition}[\citet{DemaineEFI06}]
\label{alg:demain}
There is a polynomial-time algorithm for MinDis on general weighted graphs
with approximation ratio $O(\log n)$.
\end{proposition}

\begin{proposition}[\citet{Swamy04}]
\label{alg:swamy}
There is a polynomial-time algorithm for MaxAgr on weighted graphs possibly having two parallel edges (one positive and one negative) between each pair of vertices.
This algorithm achieves an approximation ratio of $0.7666$ and
always produces a clustering into at most 6 clusters.
\end{proposition}

There is a variant of the problem where, for a given parameter $k\in \mathbb{N}$,
we optimize the MinDis and MaxAgr objectives over all clusterings into at most $k$ clusters.
We denote these variants MinDis$[k]$ and MaxAgr$[k]$ respectively.

\begin{proposition}[\citet{GiotisG06}]
\label{alg:giotis}
For constant $\gamma>0$, there are polynomial-time algorithms
for MinDis$[k]$ and MaxAgr$[k]$ on unweighted complete graphs achieving
approximation ratio of $(1+\gamma)$ and $(1-\gamma)$ respectively.
\end{proposition}

\begin{proposition}[\citet{Swamy04}]
\label{alg:swamy2}
There is a polynomial-time algorithm for MaxAgr$[k]$ on general weighted graphs
with approximation ratio $0.7666$.
\end{proposition}

For general and weighted graphs,
\citet{GiotisG06} propose an $O(\sqrt{\log n})$-approximation for
MinDis$[2]$ and show that MinDis$[k]$ is inaproximable for $k>2$.

\shortonly{\vspace{-2mm}}
\subsection{Differential privacy}
Differential privacy was first defined by \citet{DworkMNS06}. We refer the reader to  \citet{DworkR14} for a textbook treatment.

\begin{definition}[Neighboring graphs]
Let $G,G'$ be two weighted graphs on the same vertex set $V$ with weights
$w, w' \in \R^{\binom{V}{2}}$ and sign labels $\sigma,\sigma' \in \{-1,+1\}^{\binom{V}2}$.
We say that $G$ and $G'$ are {\em neighboring}, if
\[ \sum_{e\in\binom{V}{2}} |\sigma_e w_e - \sigma'_e w'_e| \leq 2. \]
\end{definition}
\shortonly{\vspace{-2mm}}
This is equivalent to switching the sign of a single edge in an unweighted graph.
In weighted graphs, an edge with a different label in $G$ and $G'$
may contribute only a small amount to the total difference, if both
labels were acquired using measurements with a low confidence
(i.e., $w_e$ and $w'_e$ are small).

\begin{definition}[Differential privacy]
Let $\ALG$ be a randomized algorithm whose domain is the set of all weighted graphs
with edges labeled by $\pm1$.
Let $\mu_G$ denote the distribution over possible outputs
of $\ALG$ given input graph $G$.
We say that $\ALG$ is {\em $(\epsilon,\delta)$-differentially private}, if
the following holds:
For any measurable $S \subseteq Range(\ALG)$ and any pair of neighboring
graphs $G$ and $G'$, we have
\[ \mu_G(S) \leq \exp(\epsilon) \mu_{G'}(S) + \delta. \]
If $\ALG$ fulfils this definition with $\delta=0$, we call it
{\em $\epsilon$-differentially private}.
\end{definition}

In other words, the output distributions of $\ALG$ on two neighboring graphs
are very similar. This implies that the output distributions are very similar
for any pair of graphs which are relatively close to each other, as
shown in the following proposition.

\begin{proposition}[Group privacy]
\label{prop:group_priv}
Let $\ALG$ be an $\epsilon$-differentially private algorithm.
Then, for any $G$ and $G'$ with distance $k$, i.e., such that
\[ \sum_{e\in\binom{V}{2}} |\sigma_e w_e - \sigma'_e w'_e| \leq 2k, \]
we have
\[ \mu_G(S) \leq \exp(k\epsilon) \mu_{G'}(S) \]
for any measurable $S \subseteq Range(\ALG)$.
\end{proposition}

An important property of differential privacy is robustness to post-processing, i.e., applying a function which does not have access to the
private data cannot make the output of $\ALG$
less differentially private.

\begin{proposition}[Post-processing]
\label{prop:post}
Let $\ALG$ be an $(\epsilon,\delta)$-differentially private algorithm
and let $f$ be an arbitrary randomized function whose domain is $Range(\ALG)$.
Then, the composition $f \circ \ALG$
is $(\epsilon,\delta)$-differentially private.
\end{proposition}

\section{Linear lower bound for paths}
\label{sec:lb}
In this section, we prove a lower bound on the additive error of $\epsilon$-DP algorithms for clustering paths.
Let $\sigma \in \{-1,+1\}^n$ be a sign vector and
$P_n(\sigma)$ denote a path on $n+1$ vertices $v_0, \dotsc, v_n$
with $n$ edges, such that the label of the edge $v_{i-1}v_i$
is $\sigma_i$ for $i=1,\dotsc, n$.
We use the following simple fact.

\begin{lemma}\label{lem:path}
Let $\sigma, \sigma' \in \{-1,+1\}^n$ be two sign vectors. The following hold:
\shortonly{\vspace{-2mm}}
\begin{enumerate}
\item For any sign vector $\sigma \in \{-1,+1\}^n$, there is an optimal clustering of $P_n(\sigma)$ with error 0.
\item If $\sigma$ and $\sigma'$ differ in at least $d$ coordinates,
no clustering can have less than $d/2$ disagreements
on both $P_n(\sigma)$ and $P_n(\sigma')$.
\end{enumerate}
\end{lemma}
\begin{proof}
To show the first statement, consider a clustering $\C$ whose clusters are
formed by vertices adjacent to sequences of positive edges in the path, as in
Figure~\ref{fig:path-opt}.
Then the endpoints of any positive edge belong to the same cluster
while endpoints of any negative edge belong to different clusters,
implying that $\err(\C, P_n(\sigma)) = 0$.

To show the second statement, let $D$ denote the set of edges with different
sign in $\sigma$ and $\sigma'$ and let $\C$ be an arbitrary clustering.
For any edge $e\in D$, the endpoints of $e$ either belong to the same
cluster in $\C$ or belong to different ones.
Since $\sigma(e) \neq \sigma'(e)$, $\C$ disagrees with $e$ either in
$P_n(\sigma)$ or $P_n(\sigma')$. By the pigeonhole principle,
$\C$ has at least $|D|/2$ disagreements with either $P_n(\sigma)$ and $P_n(\sigma')$.
\end{proof}

\begin{figure}
    \centering
    \includegraphics[width=\linewidth]{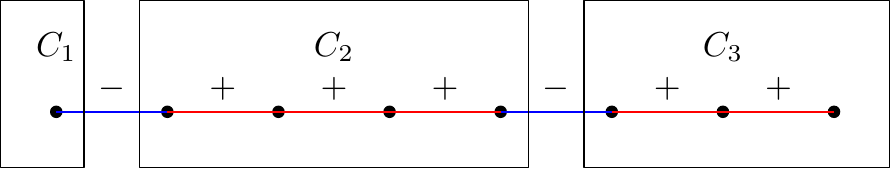}
    \caption{Optimal clustering of a path}
    \label{fig:path-opt}
\end{figure}

\paragraph{Asymptotically good codes.}
We use the following terminology from coding theory.
Let $A \subseteq \{0,1\}^n$ be a code consisting of $M$ codewords
of length $n$
and let $\alpha, \beta \in [0,1]$ be constants.
We say that $A$ has rate $\alpha$ and minimum relative distance $\beta$
if $M = 2^{\alpha n}$ and every pair of distinct codewords $c,c' \in A$ differ in
at least $\beta n$ coordinates.
Consider a family of codes $\A = \{A_i | i \in \mathbb{N}\}$,
where $A_i$ has length $n_i$, for $n_i \geq n_{i-1}$,
rate $\alpha_i$, and minimum relative distance $\beta_i$.
We say that $\A$ is \emph{asymptotically good} if its rate
$R(\A) = \lim\inf_{i} \alpha_i$ and
its minimum relative distance $d(\A) = \lim\inf_i \beta_i$ are both
strictly positive.
The following theorem proves the existence of such code families.

\begin{proposition}[Asymptotic Gilbert-Varshamov bound]
For any $\beta \in [0,1/2)$, there is an infinite family $\A$
of codes with minimum relative distance $\beta$
with rate
\[ R(\A) \geq 1 - h(\beta) - o(1), \]
where $h(\beta) = \beta \log_2 \frac1\beta + (1-\beta)\log_2 \frac1{1-\beta}$
is a constant smaller than 1 for the given $\beta$.
\end{proposition}
See, e.g., \citep{AlonBNNR92} for a construction of such codes.

\begin{corollary}\label{cor:agc}
Given a constant $\beta \in [0,1/2)$ and $n$ large enough,
there is a binary code whose codewords have pairwise distance at least
$\beta n$ of size larger than $2^{\alpha n}$ for some constant $\alpha$ depending
only on $\beta$.
In particular, for $\beta=0.1$, we can choose $\alpha=0.4$.
\end{corollary}

\shortonly{\vspace{-4mm}}
\paragraph{Lower bound construction.}
\begin{theorem}
Let $\epsilon > 0$ be a constant and $\ALG$ be a fixed
$\epsilon$-DP algorithm for MinDis.
Then the expected additive error of $\ALG$ on weighted paths
is $\Omega(n/\epsilon)$.
If $\epsilon \leq 0.2$, then its expected error is $\Omega(n)$
already on unweighted paths.
\end{theorem}

Since $\agr(\C,P_n) = n-\err(\C,P_n)$ for any $\C$,
the same error bound holds also for MaxAgr.
\shortonly{\vspace{-2mm}}
\begin{proof}
Let $\alpha=0.4/\log_2 e$ and $\beta=0.1$.
By Corollary~\ref{cor:agc},
there is $n\in \mathbb{N}$ and a code $A \subseteq\{0,1\}^n$
of size larger than $\exp(\alpha n) = 2^{0.4n}$ and minimum distance $\beta n$.
We use this code to construct a family of sign vectors
$\Sigma \subseteq \{-1,+1\}^n$ of the same size such that two distinct sign
vectors $\sigma, \sigma' \in \Sigma$ differ in at least $\beta n$ coordinates:
for any $c\in A$, we add a vector $\sigma$
to $\Sigma$, where $\sigma(e_i)$ is $+1$ whenever $c_i = 1$ and $-1$ whenever
$c_i = 0$.

Let $\lambda$ denote the weight of all the edges in $P_n$ which will be chosen later.
Then, for two distinct $\sigma, \sigma' \in \Sigma$, the distance between
input graphs $P_n(\sigma)$ and $P_n(\sigma')$ is at least $\lambda\beta n$.
We denote $B_\sigma$ a set of clusterings with error less than $\lambda\beta n/2$ on
$P_n(\sigma)$. By Lemma \ref{lem:path}, the sets $B_\sigma$ and $B_{\sigma'}$ are disjoint
for distinct $\sigma, \sigma' \in \Sigma$.

We perform a standard packing argument as in \citet{hardt2010on}.
Let $\mu_\sigma$ denote the probability measure
over the outputs of the algorithm $\ALG$ given the input $\sigma$.
For the expected error of the algorithm to be less than $\lambda\beta n/2$, we need to have
$\mu_\sigma(B_\sigma) \geq 1/2$ for any $\sigma$.
Let us fix an arbitrary input graph $P_n(\sigma)$.
The distance between $P_n(\sigma)$ and any $P_n(\sigma')$ is at most
$\lambda n$ (sum of weights of all edges). Therefore, by group privacy (Proposition~\ref{prop:group_priv}),
we have
\[ \mu_\sigma(B_{\sigma'}) \geq \mu_{\sigma'}(B_{\sigma'}) \cdot \exp(-\epsilon \cdot \lambda n)
    \geq \frac12 \exp(-\epsilon \cdot \lambda n) \]
for any $\sigma'$.
On the other hand, since the sets $B_\sigma$ are disjoint, we have
\begin{align*}
1 &\geq \mu_\sigma\big(\bigcup_{\sigma'\in \Sigma} B_{\sigma'}\big)
    \geq \frac12 + \sum_{\sigma'\in \Sigma\setminus\{\sigma\}} \mu_\sigma(B_{\sigma'})\\
    &\geq \frac12 + (|\Sigma|-1) \cdot \frac12 \exp(-\epsilon \lambda n)\\
    &\geq \frac12 + \exp(\alpha n) \cdot \frac12 \exp(-\epsilon \lambda n).
\end{align*}
If $\epsilon<\alpha$, we already achieve a contradiction
with $\lambda = 1$, showing that the expected error of the algorithm on unweighted graphs
is at least $\beta n/2$.
To get a stronger error bound for weighted graphs, we choose
$\lambda = \alpha/2\epsilon$ in order to get error $\epsilon^{-1} \alpha\beta n/4$.
\end{proof}

\section{Synthetic graph release for correlation clustering}
\label{sec:synth}
We describe mechanisms for complete unweighted graphs and for weighted (or incomplete) graphs.
They are based on existing graph release mechanisms
by \citet{GRU12} and \citet{2019_DP-cut}.

\subsection{Unweighted complete graphs}
\label{sec:synth_u}
For unweighted graphs, we
describe a graph release mechanism based on the result of \citet{GRU12} which
preserves the number of agreements and disagreements of any correlation clustering up to an additive
error of $O(k n^{3/2})$,
where $k$ is the number of clusters in the clustering.

The mechanism of \citet{GRU12} works by adding  independent Laplace noise to the weight of each edge.
See Algorithm~\ref{alg:gru} for details.

\begin{algorithm}
\caption{Release of unweighted graphs \citep{GRU12}}
\label{alg:gru}
\begin{algorithmic}
\STATE Input: $G$ with $w_e \in \{0,1\} \: \forall e\in \binom{V}{2}$
\FORALL{$e \in \binom{V}{2}$}
    \STATE $\zeta_e \sim \Lap(1/\epsilon)$
    \STATE $w_e' = w_e + \zeta_e$
\ENDFOR
\STATE Release graph with weights $w'$
\end{algorithmic}
\end{algorithm}

Given a graph $G$ on a vertex set $V$, we denote by $w_G$ the weights of its edges
where edges absent in $G$ have weight $0$.
For any $F \subseteq \binom{V}{2}$ and $S,T\subseteq V$,
we define $w_G(F) = \sum_e w_G(e)$ and $w_G(S,T) = \sum_{u\in S, v\in T} w_G(uv)$.

\begin{proposition}[\citet{GRU12}]
\label{prop:GRU12}
Algorithm~\ref{alg:gru} is $\epsilon$-differentially private,
runs in polynomial time, and given an input graph $G$,
outputs a weighted graph $H$ such that
\[ \E[w_{H}(F)] = w_G(F) \]
for any $F \subseteq \binom{V}{2}$.
Moreover, the following bound holds with high probability
for all $S,T \subseteq V$ simultaneously:
\[ |w_G(S,T) - w_H(S,T)| \leq \tilde{O}(\epsilon^{-1} n^{3/2}). \]
\end{proposition}

In our notation, $\tilde{O}$ hides terms polylogarithmic in $n$ and
is only needed for the high-probability result.
This mechanism produces a weighted graph with potentially negative weights.
\citet{GRU12} also describe a postprocessing procedure which produces
an unweighted graph with the same guarantees.

Our mechanism splits the original graph $G$ into subgraphs
$G^+$ and $G^-$ on the same vertex set containing
all positive and negative edges respectively.
We release these two graphs using Algorithm~\ref{alg:gru}.
Since the resulting graphs $H^+$ and $H^-$ may overlap and contain
edges of negative weight, we use a postprocessing step described below
to merge them into a single unweighted graph $H$.
See Algorithm~\ref{alg:synth_u} for an overview.

\begin{algorithm}
\caption{Release of unweighted complete graphs}
\label{alg:synth_u}
Split $G$ into $G^+$ and $G^-$\\
Release weighted $H^+$ and $H^-$ using Algorithm~\ref{alg:gru}\\
Merge $H^+$ and $H^-$ using the postprocessing step
\end{algorithm}

\paragraph{Postprocessing step.}
We adapt the procedure proposed by \cite{GRU12}.
Let $W^+$ and $W^-$ be the adjacency matrices of $H^+$ and $H^-$ respectively.
We formulate the linear program in Figure~\ref{fig:synth_u_LP}.
This LP has exponential number of constraints
and can be solved up to a constant factor in polynomial time
using the algorithm of \citet{AlonN06} as a separation oracle.

\begin{figure}
\begin{align*}
\min \lambda, \quad\text{ s. t. }\\
\big|\sum_{e\in S\times T} x_e - \sum_{e\in S\times T} W^+_e\big| &\leq \lambda &\forall S,T&\\
\big|\sum_{e\in S\times T} (1-x_e) - \sum_{e\in S\times T} W^-_e\big| &\leq \lambda &\forall S,T&\\
x_{e} \in [0,1] &\quad\forall e \in \binom{V}{2}&
\end{align*}
\caption{Postprocessing LP}\label{fig:synth_u_LP}
\end{figure}

Having a solution $x$ to the LP, we construct the output graph $H$.
$H$ is a complete unweighted graph and we
label its edges in the following way:
for any $e\in \binom{V}{2}$, we
label it positive with probability $x_e$ and negative otherwise.
The following fact, e.g., in \citet{Vershynin18} will be useful to analyse the properties of the
resulting graph.

\begin{proposition}[Hoeffding inequality for bounded random variables]
\label{prop:hoeffding}
Let $X_1, \dotsc, X_N$ be independent random variables such that $X_i \in[0,1]$ for each
$i=1,\dotsc, N$. For $S_N = \sum_{i=1}^N X_i$ and any $t>0$, we have
\[ P(|S_N - \E[S_N]| \geq t) \leq 2 \exp(-2t^2/N). \]
\end{proposition}

\begin{lemma}
\label{lem:synth_u}
Let $w_G^+(S,T)$ and $w_G^-(S,T)$ denote the number of positive
and negative edges respectively between the vertex sets $S$ and $T$ in graph $G$.
With high probability, we have
\begin{align*}
|w_H^+(S,T) - w_G^+(S,T)| &\leq \tilde{O}(\epsilon^{-1} n^{3/2}) \text{ and}\\
|w_H^-(S,T) - w_G^-(S,T)| &\leq \tilde{O}(\epsilon^{-1} n^{3/2})
\end{align*}
for any $S, T \subseteq V$.
\end{lemma}
\begin{proof}
By definition of $G^+$ and $G^-$, we have $w_G^+(S,T) = w_{G^+}(S,T)$ and
$w_G^-(S,T) = w_{G^-}(S,T)$.
Proposition~\ref{prop:GRU12} implies that
\begin{align*}
|w_{G^+}(S, T) - w_{H^+}(S,T)| &\leq \tilde{O}(\epsilon^{-1} n^{3/2}) \quad\text{ and}\\
|w_{G^-}(S, T) - w_{H^-}(S,T)| &\leq \tilde{O}(\epsilon^{-1} n^{3/2}).
\end{align*}

Moreover, by construction of the graph $H$, we have
\begin{align*}
\E[w_H^+(S,T)] &=\textstyle \sum_{e\in S\times T} x_e \quad\text{ and }\\
\E[w_H^-(S,T)] &=\textstyle \sum_{e\in S\times T} (1-x_e).
\end{align*}
Note that $\E[w_H^+(S,T)]$ differs from $w_{H^+}(S,T)$ by at most $\lambda^*$,
and the same holds for $\E[w_H^-(S,T)]$ and $w_{H^-}(S,T)$,
where $\lambda^*$ is the optimal value of the LP in Figure~\ref{fig:synth_u_LP}.
We claim that $\lambda^*$ is at most $O(n^{3/2})$, since graph $G$ satisfies
all the constraints for $\lambda = \tilde{O}(n^{3/2})$.
Note however that $G$ is not used when solving this LP.

Using Proposition~\ref{prop:hoeffding} with $N = |S||T|\leq n^2$,
we can show that $w_H^+(S,T)$ deviates from its expectation
by at most $n^{3/2}\log n$ with probability at least
$1 - 2\exp(-2n\log n)$.
The same holds for $w_H^-(S,T)$.
Therefore, by union bound and using the preceding relations, the following holds
\begin{align*}
|w_H^+(S,T) - w_G^+(S,T)| &\leq \tilde{O}(\epsilon^{-1} n^{3/2}) \text{ and}\\
|w_H^-(S,T) - w_G^-(S,T)| &\leq \tilde{O}(\epsilon^{-1} n^{3/2})
\end{align*}
for all $S,T\subseteq V$ at the same time with high probability.
\end{proof}

\begin{theorem}\label{thm:synth_u}
Algorithm~\ref{alg:synth_u} is $\epsilon$-differentially private
and runs in polynomial time.
Give input graph $G$, it produces graph $H$, such that for any
clustering $\C = \{C_1, \dotsc, C_k\}$, we have
\begin{align*}
|\err(\C, G) - \err(\C, H)| &\leq \tilde{O}(\epsilon^{-1} k  n^{3/2}) \text{ and}\\
|\agr(\C, G) - \agr(\C, H)| &\leq \tilde{O}(\epsilon^{-1} k  n^{3/2})
\end{align*}
with high probability.
\end{theorem}
\begin{proof}
Since we do not use $G$ in the post-processing part,
the privacy of the algorithm follows from Theorem~\ref{alg:gru} and post-processing (Proposition~\ref{prop:post}).
It runs in polynomial time, since each step can be implemented in polynomial time.

We can express the number of disagreemnents between $\C$ and $G$ as
\[
\err(\C, G) = \sum_{i=1}^k \big( w_G^-(C_i, C_i) + w_G^+(C_i, V\setminus C_i) \big).
\]
Similarly, we can express the number of agreements:
\[
\agr(\C, G) = \sum_{i=1}^k \big( w_G^+(C_i, C_i) + w_G^-(C_i, V\setminus C_i) \big).
\]

Using Lemma~\ref{lem:synth_u}, we can bound both
$|\err(\C, G) - \err(\C, H)|$ and $|\agr(\C,G) - \agr(\C,G)|$ by
$\tilde{O}(\epsilon^{-1} k n^{3/2})$.
\end{proof}

\subsection{Weighted and incomplete graphs}
\label{sec:synth_w}
We describe a graph release mechanism for weighted graphs which
preserves the cost of any correlation clustering up to an additive
error of $O(k\sqrt{mn})$,
where $k$ is the number of clusters in the clustering.
It is based on the graph release mechanism
by \citet{2019_DP-cut}.
For graphs $G$ and $H$ on the same vertex set $V$,
we define the {\em cut distance} between them as follows:
\[ d_{\cut}(G, H) = \max_{S,T\subseteq V} |w_G(S, T) - w_{H}(S, T)|, \]
Note that the sets $S$ and $T$ in the definition can be overlapping and even identical.

\begin{proposition}[\citet{2019_DP-cut}]
\label{prop:DP-cut}
Let $\cG$ be the class of weighted graphs with sum of edge weights at most $m$.
For $0 \leq \epsilon \leq 1/2$ and $0 \leq \delta \leq 1/2$, there is an $(\epsilon, \delta)$-differentially private mechanism
wich runs in polynomial time and,
for any $G \in \cG$, outputs a weighted graph $H$ such that the following
holds:
\[\textstyle
\E[ d_{\cut}(G, H)]
        \leq O\left(\sqrt{\frac{mn}{\epsilon}}\log^{2}(\frac{n}{\delta})\right).
\]
\end{proposition}

Note that the edges of the output graph $H$ have always non-negative weights.

Given an input graph $G$ whose edge weights sum up to $m$,
let $G^+$ and $G^-$ be its subgraphs containing only edges with positive and negative sign
respectively.
We output a weighted graph $H$ with possible parallel edges
which consists of edges of $H^+$ with a positive sign
and edges of $H^-$ with a negative sign.

\begin{algorithm}
\caption{Release of weighted graphs}
\label{alg:synth_w}
Split $G$ into $G^+$ and $G^-$\\
Release $H^+$ and $H^-$ using algorithm in Theorem~\ref{thm:synth_w}\\
Output union $H = H^+ \cup H^-$
\end{algorithm}

\begin{theorem}\label{thm:synth_w}
Algorithm~\ref{alg:synth_w} is $(\epsilon,\delta)$-differentially private
and runs in polynomial time.
Given input graph $G$, it produces graph $H$, such that for any
clustering $\C = \{C_1, \dotsc, C_k\}$, we have
\begin{align*}
\E[|\err(\C, G) - \err(\C, H)|] &\textstyle \leq k \cdot
O\left(\sqrt{\frac{mn}{\epsilon}}\log^{2}(\frac{n}{\delta})\right) \text{ and}\\
\E[|\agr(\C, G) - \agr(\C, H)|] &\textstyle \leq k \cdot
O\left(\sqrt{\frac{mn}{\epsilon}}\log^{2}(\frac{n}{\delta})\right).
\end{align*}
\end{theorem}
\begin{proof}
The privacy properties of the graph $H$ follow
from Propositions \ref{prop:DP-cut} and \ref{prop:post}.
Moreover, all steps of the algorithm can be implemented in polynomial time.

The disagreement between $\C$ and $G$ can be expressed as
\[
\err(\C, G) = \sum_{i=1}^k \big( w_{G^-}(C_i, C_i) + w_{G^+}(C_i, V\setminus C_i) \big).
\]
Similarly, the agreement between $\C$ and $G$ can be written as
\[
\agr(\C, G) = \sum_{i=1}^k \big( w_{G^+}(C_i, C_i) + w_{G^-}(C_i, V\setminus C_i) \big).
\]
Therefore, both $|\err(\C, G) - \err(\C, H)|$ and $|\agr(\C,G) - \agr(\C,H)|$
are bounded by
\[ k \big(d_{\cut}(G^-, H^-) + d_{\cut}(G^+, H^+)\big).\]
Together with Proposition~\ref{prop:DP-cut}, this concludes the proof.
\end{proof}

%
%

\section{Differentially private algorithms for correlation clustering.}
\label{sec:alg}
We produce a private correlation clustering as follows:
\begin{enumerate}
\shortonly{\vspace{-3mm}}
\item Release a synthetic graph $H$ using one of the differentially private
    mechanisms in Section~\ref{sec:synth}.
    \shortonly{\vspace{-1mm}}
\item Find an approximately optimal clustering $\C$ of $H$
    using some non-private approximation algorithm.
\end{enumerate}

The following simple observation will be useful later.
\begin{observation}
\label{obs:gen-alg}
Let $G$ be an input graph and $\C^*$ its optimum clustering.
Let $H$ be a graph such that
for any clustering $\C$ we have $|\err(\C,H) -\err(\C,G)| \leq \eta(|\C|)$
for some function $\eta$.
If $\C'$ is an $\alpha$-approximation to MinDis on $H$,
then $\err(\C',G) \leq \alpha \err(\C^*,G) + \eta(|\C'|) + \eta(|\C^*|)$.

Similarly, if we have $|\agr(\C,H) -\agr(\C,G)| \leq \eta(|\C|)$
for any $\C$ and $\C'$ is an $\alpha$-approximation to MaxAgr on $H$,
then we have
$\agr(\C',G) \geq \alpha \agr(\C^*,G) - \eta(|\C'|) - \eta(|\C^*|)$.
\end{observation}
\begin{proof}
For the optimal solution to MinDis on $H$, we have
\begin{align*}
\err(\C^*_H, H) \leq \err(\C^*, H) \leq \err(\C^*,G) + \eta(|\C^*|).
\end{align*}
On the other hand,  we can bound the disagreement of $\C$ as
\[ \err(\C',G) \leq \err(\C',H) + \eta(|\C'|) \leq \alpha \err(\C^*_H,H) + \eta(|\C'|).
\]
Combining these two relations concludes the proof for MinDis. The proof
for MaxAgr is analogous.
\end{proof}

\subsection{MinDis on unweighted complete graphs}
For unweighted complete graphs, we can achieve sub-quadratic additive
error for an arbitrary number of clusters.
We release $G$ using Algorithm~\ref{alg:synth_u},
letting $H$ denote its output. Now, we find an $2.06$-approximate
solution to MinDis on $H$ using the algorithm by
\citet{ChawlaMSY15}.
If $|\C| \leq n^{1/4}$, we output $\C$. Otherwise,
we transform $\C$ into a clustering of $k' = n^{1/4}$ clusters
by packing the clusters smaller than $n/k'$ into bins of at most $2n/k'$
vertices and merging each bin into a single cluster.
See Algorithm~\ref{alg:CC_u} for details.

\begin{algorithm}
\caption{DP Correlation Clustering for unweighted complete graphs}
\label{alg:CC_u}
\begin{algorithmic}
\STATE $H =$ Released synthetic graph using Algorithm~\ref{alg:synth_u}
\STATE $\C =$ $2.06$-approximate solution to MinDis on $H$
\IF{$|\C| \leq k'$, where $k'=n^{1/4}$}
    \STATE Output $\C$
\ENDIF
\STATE $\C_S =$ clusters in $\C$ of size smaller than $n/k'$
\STATE $\B =$ packing of $\C_S$ into bins of at most $2n/k'$ vertices
\FORALL {$B\in \B$}
    \STATE $C_B =$ merged clusters in the bin $B$
\ENDFOR
\STATE Output $(\C\setminus \C_S)\cup \{C_B; B\in\B\}$
\end{algorithmic}
\end{algorithm}

\begin{theorem}
\label{thm:mindis_u}
Let $G$ be an unweighted complete graph and $\C^*$ be the optimal solution
to MinDis on $G$.
Algorithm \ref{alg:CC_u} is $\epsilon$-DP, runs in polynomial time, and
finds a clustering $\C$ such that
\[ \err(\C, G) \leq 2.06 \cdot \err(\C^*,G)
    + \tilde{O}(\epsilon^{-1}n^{1.75}).
\]
\end{theorem}
\begin{proof}
The privacy properties of the algorithm follow from the privacy of Algorithm~\ref{alg:synth_u}
and the post-processing rule (Proposition~\ref{prop:post}).
Algorithm runs in polynomial time, since all steps, including the packing, which can be done
greedily, can be implemented in polynomial time.

If $|\C| \leq n^{1/4}$, we choose $\C' = \C$. Otherwise,
we transform $\C$ into a clustering $\C'$ of $k' = n^{1/4}$ clusters
in the following way.
All clusters in $\C$ of more than $n/k'$ vertices remain separate clusters
and the smaller ones are packed into bins of at most $2n/k'$ vertices.
Merging of each bin into one cluster introduces error due to negative edges of at most
$(2n/k')^2$, with all $k'$ bins causing the total error of $O(n^2/k')$.
The total number of clusters in $\C'$ is at most $k'$.
Therefore, by Proposition~\ref{alg:chawla}, we have
\[ \err(\C', H) \leq 2.06 \cdot \err(\C^*_H,H) + O(n^2/k'),
\]
where $\C^*_H$ is the optimal clustering of the graph $H$.
Using the same argumentation,
we also get $\err(\C'_G,G) \leq \err(\C^*_G, G) + n^2/k'$,
for the best clustering $\C'_G$ of $G$ into $k'$ clusters.
We have
\[ \err(\C^*_H, H) \leq \err(\C'_G,H) \leq \err(\C'_G, G) + \tilde{O}(\epsilon^{-1} k' n^{3/2}) \]
by optimality of $\C^*_H$ and Theorem~\ref{thm:synth_u}.
By combining the preceding relations and using Theorem~\ref{thm:synth_u} once more,
we finally get
\[ \err(\C', G) \leq 2.06 \err(\C^*_G, G) + \tilde{O}(\epsilon^{-1}k' n^{3/2} +  n^2/k'), \]
where the last term can be bounded by $\tilde{O}(\epsilon^{-1}n^{1.75})$.
\end{proof}

\subsection{MaxAgr on unweighted complete graphs}
The algorithm is the same as Algorithm~\ref{alg:CC_u}, except for $\C$ being
the approximate solution to MaxAgr problem on $H$.
We find $\C$ using the PTAS by \citet{BansalBC02} (Proposition~\ref{alg:bansal}).
The analysis follows the same lines as the proof of Theorem~\ref{thm:mindis_u}.
Note that the loss in the objective due to the packing of small clusters into
bins is the same as in the case of MinDis: it is the number of negative
edges between the clusters packed in the same bin.

\begin{theorem}
Let $G$ be an unweighted complete graph and $\C^*$ be optimal solution
to MaxAgr on $G$. For any $\gamma>0$, there is an $\epsilon$-DP algorithm which
runs in polynomial time and
finds a clustering $\C$ such that
\[ \agr(\C, G) \leq (1+\gamma) \cdot \agr(\C^*,G)
    - \tilde{O}(\epsilon^{-1}n^{1.75}).
\]
\end{theorem}

\subsection{Algorithms for weighted and incomplete graphs}
We use Algorithm~\ref{alg:synth_w} as a release mechanism, whose output
may have two weighted edges between a single pair of vertices: one with
a positive and one with a negative label.

\paragraph{Minimizing disagreement.}
Given input graph $G$, we release its approximation $H$
using Algorithm~\ref{alg:synth_w}.
We split each vertex $v$ of $H$ into two vertices
$v^+$ and $v^-$, attaching the positive edges adjacent to $v$
to $v^+$ and negative ones to $v^-$. We connect $v^+$ and $v^-$
by an edge of infinite weight to ensure they remain in the same cluster.
Then, we find an $O(\log n)$-approximate solution $\C$ on the modified
graph using the algorithm by \citet{DemaineEFI06} (Proposition~\ref{alg:demain}).
We eliminate duplicate vertices from $\C$ and pack clusters of less than
$n/k'$ vertices into at most $k' = n^{1/4}$ bins of size at most $2n/k'$,
like in Algorithm~\ref{alg:CC_u}.
See Algorithm~\ref{alg:mindis_w} for details.

\begin{algorithm}
\caption{MinDis on weighted graphs}
\label{alg:mindis_w}
\begin{algorithmic}
\STATE $H =$ Released synthetic graph using Algorithm~\ref{alg:synth_w}
\FORALL{$v\in V(H)$}
    \STATE add vertices $v^+$ and $v^-$ to $H'$
    \STATE add edge $v^+v^-$ to $H'$ with weight $+\infty$
\ENDFOR
\FORALL{$e\in E(H)$}
    \STATE $u,v$ be endpoints of $e$; $\sigma=$ sign of $e$; $w=$ weight of $e$
    \STATE add edge $v^{\sigma}u^{\sigma}$ with weight $w$ to $H'$
\ENDFOR
\STATE $\C =$ $O(\log n)$-apx solution on $H'$.
\STATE $\C'=$ merge duplicate vertices in $\C$, pack small clusters\\  
    $\quad$ into bins, and merge each bin into a single cluster
\STATE Output $\C'$
\end{algorithmic}
\end{algorithm}

\begin{theorem}
\label{thm:mindis_w}
Let $G$ be a weighted graph such that $W$ is the weight of its heaviest edge,
and the sum of its edge weights is at most $m\leq O(n^2)$.
Let $\C^*$ be the optimal solution to MinDis on $G$.
Algorithm~\ref{alg:mindis_w} is $(\epsilon,\delta)$-DP, runs in polynomial time,
and finds a clustering $\C$ such that
\[ \err(\C,G) \leq O(\log n)\cdot \err(\C^*,G) + \beta,
\]
where $\beta = O\big(W n^{1.75} \cdot \epsilon^{-\frac12} \log^2(n/\delta)\big)$.
\end{theorem}
\begin{proof}
The optimum solution on $H'$ is the same as on $H$, since any optimum solution
has to put $v^-$ and $v^+$ to the same cluster, for any vertex $v$ of $H$.
Therefore, if $\C$ is an $O(\log n)$-apx solution on $H'$, it
has to be an $O(\log n)$-apx solution also on $H$.

Due to packing of small clusters, we misclassify at most $n^2/k'$ negative
edges, each of them of weight at most $W$.
Since the additive error due to the release of the original graph $G$
using Algorithm~\ref{alg:synth_w} is $O(\sqrt{mn/\epsilon}\log^2(n/\delta)$,
the total additive error is
$O\big(W n^{1.75} \cdot \epsilon^{-\frac12} \log^2(n/\delta)\big)$.
\end{proof}

\shortonly{
\paragraph{Additional results in supplementary material.}
In supplementary material, we present an algorithm with additive error $\tilde{O}(\sqrt{mn/\epsilon})$
for MaxAgr on weighted graphs (Theorem 28), and results for MaxAgr$[k]$ and MinDis$[k]$
where we need to find the best clustering into at most $k$ clusters,
see Section 5.4.
}

\fullonly{
\paragraph{Maximizing agreement.}
Again, we release the input graph $G$ using Algorithm~\ref{alg:synth_w}.
However, we do not need to do any postprocessing, since
the algorithm of \citet{Swamy04} supports graphs with
one positive and one negative weighted edge between each pair of vertices.

\begin{algorithm}
\caption{MaxAgr for weighted graphs}
\label{alg:maxagr_w}
\begin{algorithmic}
\STATE $H=$ Released input graph using Algorithm~\ref{alg:synth_w}
\STATE $\C=$ solution on $H$ found by algorithm of \citet{Swamy04}
\end{algorithmic}
\end{algorithm}

\begin{theorem}
Let $G$ be an input graph and $\C^*$ be the optimal solution
to MaxAgr on $G$.
Algorithm~\ref{alg:maxagr_w} is $(\epsilon,\delta)$-DP, runs in polynomial time,
and finds a clustering $\C$ such that
\[\textstyle
\agr(\C,G) \geq \Omega(1) \agr(\C^*,G) - O(\sqrt{\frac{mn}{\epsilon}} \log^2\frac{n}{\delta}).
\]
If $|\C^*|= k$, then we have
\[\textstyle
\agr(\C,G) \geq 0.7666 \agr(\C^*,G) - O(k\sqrt{\frac{mn}{\epsilon}} \log^2\frac{n}{\delta}).
\]
\end{theorem}
\begin{proof}
The proof of the second statement follows from Observation~\ref{obs:gen-alg} and
the fact that the algorithm of \citet{Swamy04} always returns clustering of at most 6 clusters,
see Proposition~\ref{alg:swamy}.

To show the first part, note that Proposition~\ref{alg:swamy} also implies, that there
is a solution $\C'$ of at most $k$ clusters, such that
$\agr(\C',G)\geq 0.7666 \agr(\C^*,G)$. Therefore, we have
\[ \agr(\C,G) \geq 0.7666^2 \agr(\C^*) - O(\sqrt{\frac{mn}{\epsilon}} \log^2\frac{n}{\delta}).
\qedhere
\]
\end{proof}

\subsection{Fixed number of clusters}
For a fixed $k$, we look for a clustering into $k$ clusters which
minimizes the disagreements or maximizes agreements.
This problem was studied by \citet{Swamy04} and \citet{GiotisG06}.
\shortonly{\vspace{-2mm}}
\paragraph{Unweighted complete graphs.}
There are PTAS algorithms by \citet{GiotisG06} for MinDis$[k]$ and MaxAgr$[k]$,
for a constant $k$.
We use them to find a correlation clustering of size $k$
on a graph released using Algorithm~\ref{alg:synth_u}.
The following theorems are implied by Observation~\ref{obs:gen-alg}
and Proposition~\ref{alg:giotis}.

\begin{theorem}
Let $G$ be an unweighted complete graph and let $\C^*$ be the optimal
solution to MinDis$[k]$ on $G$.
There is an $\epsilon$-DP algorithm for MinDis$[k]$
which runs in polynomial time and
produces a clustering $\C$ of size $k$, such that
\[ \err(\C,G) \leq (1+\epsilon) \err(\C^*,G) + O(kn^{3/2}). \]
\end{theorem}

\begin{theorem}
Let $G$ be an unweighted complete graph and let $\C^*$ be the optimal
solution to MaxAgr$[k]$ on $G$.
There is an $\epsilon$-DP algorithm for MaxAgr$[k]$
which runs in polynomial time and
produces a clustering $\C$ of size $k$, such that
\[ \agr(\C,G) \geq (1+\epsilon) \agr(\C^*,G) - O(kn^{3/2}). \]
\end{theorem}

\paragraph{Weighted and incomplete graphs.}
We us the algorithm by \citet{Swamy04}, which
allows two edges between each pair of vertices (one positive and one negative),
to find clustering on a graph released by Algorithm~\ref{alg:synth_w}.
The following theorem follows from Observation~\ref{obs:gen-alg} and Proposition~\ref{alg:swamy2}.

\begin{theorem}
Let $G$ be a graph and let $\C^*$ be the optimal
solution to MaxAgr$[k]$ on $G$.
There is an $(\epsilon,\delta)$-DP algorithm for MaxAgr$[k]$
which runs in polynomial time and
produces a clustering $\C$ of size $k$, such that
\[\textstyle
\agr(\C,G) \geq 0.7666 \agr(\C^*,G) - O(k\sqrt{\frac{mn}{\epsilon}} \log^2\frac{n}{\delta}). \]
\end{theorem}
}

\bibliography{main}
\bibliographystyle{icml2021}

\end{document}